\documentclass[twoside]{article}

\usepackage[accepted]{aistats2022}
\usepackage[round]{natbib}

\newcommand{\dsep}[4]{ #1 \perp_{#4} #2 \vert #3}

\usepackage{tikz, subcaption}
\usepackage{mathrsfs}
\usetikzlibrary{shapes, arrows, positioning}
\usepackage{array}
\usepackage{dsfont}
\usepackage{amsmath,amsfonts,amssymb,mathtools, amsthm}
\usepackage[noend]{algorithmic}
\usepackage[ruled,vlined,linesnumbered]{algorithm2e}
\usepackage{tikz, subcaption}
\usetikzlibrary{shapes, arrows, positioning}

\newcommand{\independent}{\perp\mkern-9.5mu\perp}

\newtheorem{theorem}{Theorem}
\newtheorem{corollary}{Corollary}
\newtheorem{lemma}{Lemma}
\newtheorem{example}{Example}

\newtheorem{definition}{Definition}

\begin{document}

%

%

\twocolumn[

\aistatstitle{Causal Effect Identification with Context-specific Independence Relations of Control Variables}

\aistatsauthor{ Ehsan Mokhtarian \And Fateme Jamshidi \And  Jalal Etesami \And Negar Kiyavash }

\aistatsaddress{ EPFL, Switzerland \And  EPFL, Switzerland \And EPFL, Switzerland \And EPFL, Switzerland} ]

\begin{abstract}
    We study the problem of causal effect identification from observational distribution given the causal graph and some context-specific independence (CSI) relations. It was recently shown that this problem is NP-hard, and while a sound algorithm to learn the causal effects is proposed in  \cite{tikka2020identifying}, no provably complete algorithm for the task exists. In this work, we propose a sound and complete algorithm for the setting when the CSI relations are limited to observed nodes with no parents in the causal graph. One limitation of the state of the art in terms of its applicability is that the CSI relations among all variables, even unobserved ones, must be given (as opposed to learned). Instead, We introduce a set of graphical constraints under which the CSI relations can be learned from mere observational distribution. This expands the set of identifiable causal effects beyond the state of the art. 
\end{abstract}

\section{INTRODUCTION}
    
    Data-driven approaches to identify a causal effect from a combination of observations, experiments, and side information about the problem of interest is central to science. 
    Causal effect identification considers whether an interventional probability distribution can be uniquely determined from available information (\cite{pearl1995causal,tian2003ID}).
    
    In the absence of unobserved variables in the system, \cite{pearl1995causal} showed that the causal graph along with the observational distribution suffices to uniquely identify all interventional distributions.
    On the other hand, when there are hidden variables, causal identifications become more challenging.
    \cite{pearl1995causal} introduced three rules, known as do-calculus, for identifying causal effect in the presence of unobserved variables. 
    \cite{shpitser2006identification} later showed that applying these rules along with probabilistic manipulations is \textit{complete} to determine whether an interventional distribution is identifiable from only observational distribution given the causal graph.
    Interestingly, when further side information about the underlying generative model is available, the completeness results of do-calculus is no longer valid. That is, although more causal effects become identifiable, do-calculus based methods fail to identify them.
    For example, consider the causal graph in Figure \ref{fig: example1}. The causal effect of $X$ on $Y$ is not identifiable from the graph, while given a set of context-specific independence (CSI) relations; it becomes identifiable (see Example \ref{ex:example2}).
    
    Statistical independence, specifically conditional independence (CI) relations between a set of random variables, plays an important role in causal inference \citep{mokhtarian2021recursive}. 
    An important generalization of this concept is CSI (\cite{boutilier1996context,shimony1991explanation}) which refers to a conditional independence relation that is true only in a specific context (See Section \ref{sec: control var} for more details). 
    Consider the following example. Smoking normally has a causal effect on blood pressure. But, when a person has a ratio of beta and alpha lipoproteins larger than a threshold, whether or not he smokes is unlikely to affect his blood pressure (\cite{edwards1985fast}). Thus, the blood pressure is independent of smoking in the context of this ratio.
    CSI relations have been used to analyze, e.g., gene expression data (\cite{barash2002context}), parliament elections, prognosis of heart disease (\cite{nyman2014stratified}), etc. 
    They have also been used to improve probabilistic inference (\cite{chavira2008probabilistic, dal2018parallel}) and structure learning (\cite{chickering1997bayesian, hyttinen2018structure}). 
    Similar to \cite{tikka2020identifying}, in this work, we study the causal effect identification problem with extra information in the form of CSI relations. Our contributions are as follows. 
    
    \begin{itemize}
        \item We study the causal identification problem in the presence of CSI relations of a subset of observed variables, called \textit{control variables}, as well as the causal graph. We show that this problem is equivalent to a series of causal effect identifications only from causal graphs (Theorem \ref{thm: main}). Consequently, we propose the first sound and complete algorithm (Algorithm \ref{algo}) for this problem. 
        
        \item We introduce a graphical constraint under which the CSI relations of control variables can be inferred only from the observational distribution (Algorithm \ref{algo 2}). This expands the set of identifiable causal effects beyond the state-of-the-art approaches. More precisely, do-calculus-based methods determine the identifiability of a causal effect only from the causal graph without utilizing the other available source of knowledge, that is the observational distribution. Algorithm \ref{algo 2} uses both the causal graph and the observational distribution to determine the identifiability. 
    \end{itemize}
    
    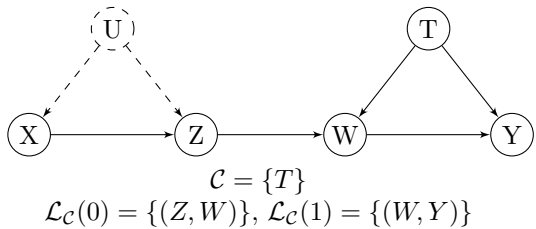
\begin{figure}[t] 
        \centering
        \tikzstyle{block} = [draw, fill=white, circle, text centered,inner sep= 0.2cm]
    	\tikzstyle{input} = [coordinate]
    	\tikzstyle{output} = [coordinate]
        \begin{subfigure}[b]{0.5\textwidth}
        \centering
	    \begin{tikzpicture}[->, auto, node distance=1.3cm,>=latex', every node/.style={inner sep=0.12cm}]
		    \node [block, dashed, label=center:U](U) {};
		    \node [block, below right= 1cm and 0.7cm of U, label=center:Z](Z) {};
    		\node [block, below left= 1cm and 0.7cm of U, label=center:X](X) {};
    		\node [block, right= 1.4cm of Z, label=center:W](W) {};
    		\node [block, above right= 1cm and 0.7cm of W, label=center:T](T) {};
    		\node [block, below right= 1cm and 0.7cm of T, label=center:Y](Y) {};
    		\draw[dashed] (U) to (Z);
    		\draw (X) to (Z);
    		\draw[dashed] (U) to (X);
    		\draw (Z) to (W);
    		\draw (T) to (W);
    		\draw (T) to (Y);
    		\draw (W) to (Y);
		\end{tikzpicture}
        \end{subfigure} 
	    \centering
	    $\mathcal{C} = \{T\}$
	    
		$\mathcal{L}_{\mathcal{C}}(0)= \{(Z,W)\},\, \mathcal{L}_{\mathcal{C}}(1)= \{(W,Y)\}$
    	\caption{Causal graph and $\mathcal{L}_{\mathcal{C}}$ in Example \ref{ex:example2}}
    	\label{fig: example1}
    \end{figure}


    To prove the completeness result of our proposed algorithm in Theorem \ref{thm: main}, we introduce a novel proof technique.
    In this technique, in order to show the identifiablity of a causal effect $P_{\mathbf{t}}(\mathbf{s})$ (see Section \ref{sec: preliminaries} for formal definition) from the observational distribution $P(\mathbf{N})$, instead of finding an exact formula for $P_{\mathbf{t}}(\mathbf{s})$ from $P(\mathbf{N})$, we show that it can be approximated by a functional of $P(\mathbf{N})$ with arbitrary accuracy. 

\subsection{Related Work}

The causal effect identification problem has been extensively studied in the literature.
Given the causal graph and the observational distribution, there are several sound and complete algorithms in the literature for identifying causal effects. 

\cite{pearl1995causal} proposed three rules, known as do-calculus, which, along with probabilistic manipulations, suffice to derive a formula for a causal effect $P_{\mathbf{t}}(\mathbf{s})$ based on observational distribution $P(\mathbf{N})$. 
Later, \cite{shpitser2006identification} showed that these rules are complete for the identification of  causal effect from $\mathcal{G}$. That is, if do-calculus rules fail to obtain a formula, the causal effect is non-identifiable.

\cite{tian2003ID} proposed another algorithm for this problem and \cite{huang2008completeness} later proved its completeness. 
In this algorithm, interventional distributions are expressed by $Q$ functions using $Q[\mathbf{s}](\mathbf{n}):=P_{\mathbf{n\setminus s}}(\mathbf{s})$. 
Subsequently, they show that the identifiablity of $P_{\mathbf{t}}(\mathbf{s})$ from $P(\mathbf{N})$ is equivalent to the identifiablity of a particular $Q$ function from $P(\mathbf{N})$ and propose an algorithm for the identifiablity of the $Q$ function from $P(\mathbf{N})$. 

There are other variants of causal identification problem in the literature. In these variants, causal effect identification problem is studied under different information sets. 
For instance, \cite{bareinboim2012causal} assumed there exists a subset of variables $\mathbf{Z}$ that could be intervened upon. 
That is, we have access to $P_{ \mathbf{z'}}(\mathbf{v} \setminus \mathbf{z'})$ for all $\mathbf{Z'} \subseteq \mathbf{Z}$, the observational distribution, and the causal graph. 
They presented a sound and complete algorithm for identification of a causal effect. 
However, in practice, we might not have access to $P_{ \mathbf{z'}}(\mathbf{v} \setminus \mathbf{z'})$ for all $\mathbf{Z'} \subseteq \mathbf{Z}$. 
\cite{lee2019general} generalized \cite{bareinboim2012causal}'s result by restricting the set of interventional distributions $P_{ \mathbf{z'}}(\mathbf{v} \setminus \mathbf{z'})$ to some $\mathbf{Z}'\subseteq \mathbf{V}$ by proposing 
a sound and complete algorithm for this setting. 
\cite{lee2020causal} studied the problem of causal effect identification when the available distributions are only \emph{partially} observable and proposed a sound algorithm to compute a causal effect in terms of the available distributions. They did not show the completeness of their algorithm. 
\cite{zhang2021bounding} considered the problem of bounding causal effects from
experiments when the assignment of treatment is randomized while the subject compliance is imperfect. Due to unobserved variables, the causal effects are not identifiable. Thus, they propose bounds over the causal effects.

\cite{tikka2020identifying} studied the causal effect identification in the presence of CSI relations, observational distribution, and the causal graph. They showed that the problem in general is NP-hard and proposed a sound algorithm for this setting, but they did not provide completeness results.

\cite{robins2020interventionist} considered the problem of identifiability for Controlled Direct Effect (CDE) that is an average causal effect when CSI relations are available. 
They showed that classical do-calculus based methods cannot determine the identifiability of CDE and proposed an algorithm that calculates CDE when additional CSI relations were available.

CSI relations have also been used for structure learning. As an example, \cite{ramanan2020causal} used CSI relations to identify the candidate set of causal relationships for learning structural causal models from observational data.
\cite{hyttinen2018structure} considered the problem of structure learning for Bayesian networks with CSI relations. They proposed orientation rules that utilizes  CSI relations to orient edges.

    \begin{table}[t]
        \caption{Table of notations} \label{tab: notation}
        \centering
        \begin{tabular}{c| c} 
            \hline 
            \textbf{Notation} & \textbf{Description}  \\
            \hline 
            $Pa_X$& Parents of $X$ \\
            $Anc_X$ & Ancestors of $X$\\
            $\Lambda_\mathcal{G}$& Set of observed roots in $\mathcal{G}$\\ 
            $\mathcal{C}$& Set of control variables \\
            $\mathcal{C}^{X}$& $ Pa_X \cap \mathcal{C}$ \\
            $D_{\mathcal{C}}$ & Domain of $\mathcal{C}$\\
            \hline 
        \end{tabular}
    \end{table}
\section{PRELIMINARIES} \label{sec: preliminaries}
    In this paper, we use capital and small letters to denote random variables and their realizations, respectively.
    Similarly, sets of random variables are denoted by bold capital letters and sets of their realizations by small bold letters. Observable and unobservable variables are represented in solid and dashed circles in the figures, respectively.
    We assume all the variables are discrete with finite domain.
 
    Let $\mathcal{G}=(\mathbf{V}, \mathbf{E})$ be a directed acyclic graph (DAG) with a finite set of vertices $\mathbf{V}$ and a set of edges $\mathbf{E}$. 
    $X\in \mathbf{V}$ is called a \emph{parent} of $Y$ if $(X,Y)\in \mathbf{E}$.
    $X$ is called an \emph{ancestor} of $Y$ if there exists a directed path from $X$ to $Y$. 
    For $X\in \mathbf{V}$, $Pa_X$ and $Anc_X$ denote the set of parents and the set of ancestors of $X$, respectively. 
    Note that $X \in Anc_X$. 
    For $\mathbf{X}\subseteq \mathbf{V}$, $Anc_{\mathbf{X}}$ denotes the union of ancestors of the variables in $\mathbf{X}$.
    We use $D_{\mathbf{X}}$ to denote the domain of $\mathbf{X}$.
    A vertex $X\in \mathbf{V}$ is called a \emph{root} if it has no parents. 
    $\Lambda_{\mathcal{G}}$ denotes the set of observed roots in $\mathcal{G}$.  
    Table \ref{tab: notation} summarizes some of the notations used in this paper.
    
    To model the environment, we assume that the variables are generated by a Structural Equation Model (SEM) \citep{pearl2009causality}. 
    In such models, each variable $X \in \mathbf{V}$ is generated as $f_X(Pa_X,\epsilon_X)$, where $f_X$ is a deterministic function and $\epsilon_X$ is the exogenous noise corresponding to $X$ such that the noise variables are jointly independent. 
    Suppose $\mathcal{M}$ is a SEM with the causal DAG $\mathcal{G}$ and the joint distribution $P^{\mathcal{M}}(\mathbf{V})$ over $\mathbf{V}= \mathbf{N} \cup \mathbf{U}$, where $\mathbf{N}$ and $\mathbf{U}$ are the set of observable and unobservable variables, respectively.
    An intervention $do(\mathbf{T=t})$ in $\mathcal{M}$ on a set $\mathbf{T}$ is defined as forcing $\mathbf{T}$ to be $\mathbf{t}$ and eliminating the impact of other variables on those in $\mathbf{T}$. 
    $P_{\mathbf{t}}(\mathbf{s}):=P(\mathbf{S=s}|do(\mathbf{T=t}))$ denotes the post-interventional distribution of $\mathbf{S}$ after the intervention $do(\mathbf{T=t})$ (\cite{pearl2009causality}).
    In the classic problem of causal effect identification, the goal is to compute $P_{\mathbf{t}}(\mathbf{s})$ from the observational distribution.

    \begin{definition}[ID from $\mathcal{G}$ \citep{pearl2009causality}] \label{def: ID from graph}
        The causal effect of $\mathbf{T}$ on $\mathbf{S}$ is said to be identifiable from $\mathcal{G}$ if for any $\mathbf{t}\in D_{\mathbf{T}}$ and $\mathbf{s} \in D_\mathbf{S}$, $P^{\mathcal{M}}_{\mathbf{t}}(\mathbf{s})$ is uniquely computable from $P^{\mathcal{M}}(\mathbf{N})$ in any SEM $\mathcal{M}$ with causal graph $\mathcal{G}$ such that $P^{\mathcal{M}}(\mathbf{n})>0$ for any $\mathbf{n} \in D_{\mathbf{N}}$.
    \end{definition}
    When an interventional distribution is identifiable, it can be uniquely computed from the joint observational distribution by a series of continuous operations (e.g., marginalization, Bayes rule, and the law of total probability). 
    Formally, \emph{uniquely computable} in Definition \ref{def: ID from graph} means that there exists a continuous operator $\mathcal{F}$ such that $\mathcal{F}(P^{\mathcal{M}}(\mathbf{N})) = P^{\mathcal{M}}_{\mathbf{t}}(\mathbf{s})$.

  
    Let $\mathbf{X}$, $\mathbf{Y}$ and $\mathbf{S}$ be three disjoint subsets of vertices in $\mathcal{G}$. $\mathbf{X}$ and $\mathbf{Y}$ are d-separated by $\mathbf{S}$, denoted by $\dsep{\mathbf{X}}{\mathbf{Y}}{\mathbf{S}}{\mathcal{G}}$, if every path between $\mathbf{X}$ and $\mathbf{Y}$ is blocked\footnote{See \cite{pearl2009causality} for definition of blocking.} by $\mathbf{S}$.    
    As we mentioned in the related work section, \cite{pearl1995causal} proposed following do-calculus rules.
   
    Rule 1 (Insertion/deletion of observations)
    \[
        P_{\mathbf{x}}(\mathbf{y} \mid \mathbf{z},\mathbf{w})=P_{\mathbf{x}}(\mathbf{y} \mid \mathbf{w}) \ \text{, if} \ \dsep{\mathbf{Y}}{\mathbf{Z}}{\mathbf{X},\mathbf{W}}{\mathcal{G}_{\overline{\mathbf{X}}}}.
    \]
    Rule 2 (Action/observation exchange)
    \[
        P_{\mathbf{x},\mathbf{z}}(\mathbf{y} \mid \mathbf{w})=P_{\mathbf{x}}(\mathbf{y} \mid \mathbf{w},\mathbf{z}) \ \text{, if} \ \dsep{\mathbf{Y}}{\mathbf{Z}}{\mathbf{X},\mathbf{W}}{\mathcal{G}_{\overline{\mathbf{X}}, \underline{\mathbf{Z}}}}.
    \]
     Rule 3 (Insertion/deletion of actions)
    \[
        P_{\mathbf{x},\mathbf{z}}(\mathbf{y} \mid \mathbf{w})=P_{\mathbf{x}}(\mathbf{y} \mid \mathbf{w}) \ \text{, if} \ \dsep{\mathbf{Y}}{\mathbf{Z}}{\mathbf{X},\mathbf{W}}{\mathcal{G}_{\overline{\mathbf{X}}, \overline{\mathbf{Z}(\mathbf{W})}}}.
    \]
    In above, $\mathcal{G}_{\overline{\mathbf{X}}}$ (or  $\mathcal{G}_{\underline{\mathbf{X}}}$) denotes the graph obtained by deleting all the in-going (or out-going) edges of the variables in $\mathbf{X}$ from $\mathcal{G}$, and $\mathbf{Z}(\mathbf{W})$ is the set of nodes in $\mathbf{Z}$ that are not ancestors of any node in $\mathbf{W}$.

\section{CONTROL VARIABLES AND IDENTIFIABILITY FROM $(\mathcal{G},\, \mathcal{L}_{\mathcal{C}})$} \label{sec: control var}
    Let $\mathbf{X},\mathbf{Y},\mathbf{S}$ be three disjoint subsets of $\mathbf{V}$.
    $\mathbf{X} \independent \mathbf{Y} \vert \mathbf{S}$ denotes the Conditional Independence (CI) of  $\textbf{X}$ and $\textbf{Y}$ given $\mathbf{S}$.
    A context-specific independence (CSI) relation is of the form $\mathbf{X} \independent \mathbf{Y} \vert \mathbf{S}_1\!=\!\mathbf{s}_1, \mathbf{S}_2$, where $\mathbf{X}, \mathbf{Y}, \mathbf{S}_1$, and $\mathbf{S}_2$ are disjoint subsets of $\mathbf{V}$ and $\mathbf{s}_1\in D_{\mathbf{S}_1}$. 
    
    
    \begin{example}\label{ex:example1}
    Consider the DAG in Figure \ref{fig: example2} with the following SEM.
    \begin{align*}
        &T = Be(0.5),&  Z = Be(0.5),\\
        &Y = TX+Z, & X = T\oplus Z,
    \end{align*}
    where $Be(q)$ and $\oplus$ denote Bernoulli distribution with parameter $q$ and logical xor, respectively. 
    In this example, $Y$ is a function of $X$ and there is no CI relation between them.
    However, the following CSI relation holds: $X \independent Y \vert T=0, Z$.
    \end{example}
    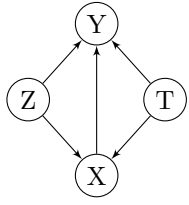
\begin{figure}[t] 
        \centering
        \tikzstyle{block} = [draw, fill=white, circle, text centered,inner sep= 0.2cm]
    	\tikzstyle{input} = [coordinate]
    	\tikzstyle{output} = [coordinate]
    	    \begin{tikzpicture}[->, auto, node distance=1.3cm,>=latex', every node/.style={inner sep=0.12cm}]
    		    \node [block, label=center:Y](Y) {};
    		    \node [block, below right= 0.6cm and 0.5cm of U, label=center:T](T) {};
        		\node [block, below left= 0.6cm and 0.5cm of U, label=center:Z](Z) {};
        		\node [block, below left= 0.6cm and 0.5cm of T, label=center:X](X) {};
        		\draw (T) to (Y);
        		\draw (Z) to (Y);
        		\draw (X) to (Y);
        		\draw (T) to (X);
                \draw (Z) to (X);
    		\end{tikzpicture}
    	\caption{Causal graph of the SEM in Example \ref{ex:example1}. }
    	\label{fig: example2}
    \end{figure}
    
    Example \ref{ex:example1} shows that, in general, CSI relations cannot be inferred merely from the causal graph. They have to be either provided as side information or inferred via additional assumptions.  
    In order to encode CSI relations succinctly into the causal graph, analogous to \cite{tikka2020identifying, pensar2015labeled}, we define a set of labels as follows.  
    \begin{definition}[Label set $\mathcal{L}_{\mathbf{C}}$]\label{def:label}
        Suppose $\mathbf{C}\subseteq \mathbf{V}$ and denote by $\mathcal{L}_{\mathbf{C}}(\mathbf{c})$ a set of edges $(Y, X)\in \mathbf{E}$ such that $X \independent$ $Y\vert \mathbf{C}=\mathbf{c}, \mathbf{S}$ for some $\mathbf{S}\subseteq \mathbf{V}\setminus\big(\{X,Y\}\cup \mathbf{C}\big)$.
        We define the label set $\mathcal{L}_{\mathbf{C}}:= (\mathcal{L}_{\mathbf{C}}(\mathbf{c})\!:\: \mathbf{c}\in D_{\mathbf{C}})$.
    \end{definition}
    Accordingly, a SEM $\mathcal{M}$ is said to be \emph{compatible} with $\mathcal{L}_{\mathbf{C}}$, if for any $(Y, X)\in \mathcal{L}_{\mathbf{C}}(\mathbf{c})$, $X$ is no longer a function of $Y$ when $\mathbf{C=c}$.
    In Example \ref{ex:example1}, $\mathcal{L}_{T}(0)=\{(X,Y)\}$ and $\mathcal{L}_{T}(1)= \varnothing$. 
    
    It is important to emphasize that \cite{tikka2020identifying} assumed all CSI relations $X \independent Y \vert \mathbf{C}=\mathbf{c}$ are given as side information, where $\mathbf{C} = Pa_X \setminus \{Y\}$. 
    This includes CSI relations with conditioning set $\mathbf{C}$ that could include unobserved variables. 
    Therefore, it is impossible to obtain all such CSI relations from merely observational distribution. 
    In this paper, we relax this assumption and restrict $\mathbf{C}$ to be a particular subset of observed variables called \textit{control variables}. 
    
    
    \begin{definition}
        The set of control variables, denoted by $\mathcal{C}$, is a subset of the observed roots, i.e., $\mathcal{C} \subseteq \Lambda_{\mathcal{G}}$.
    \end{definition}
    Knowing the label set of control variables, i.e., $\mathcal{L}_{\mathcal{C}}$, is equivalent to knowing the edges that will be omitted from $\mathcal{G}$ for different realizations of $\mathcal{C}$. For instance, in Example \ref{ex:example1}, $\mathcal{L}_{T}(0)=\{(X,Y)\}$ implies that the edge $(X,Y)$ will be deleted from the graph when $T=0$.
    We will show that such side information can be utilized to identify some causal effects that would have been non-identifiable from merely do-calculus. 

    Next, we formally define  identifiablity from $(\mathcal{G},\, \mathcal{L}_{\mathcal{C}})$
    \begin{definition}[ID from $(\mathcal{G},\, \mathcal{L}_{\mathcal{C}})$]
        Suppose $\mathbf{T,S}\subseteq \mathbf{N}\setminus \mathcal{C}$ are disjoint subsets of variables that are not control variables. 
        The causal effect of $\mathbf{T}$ on $\mathbf{S}$ is said to be identifiable from $(\mathcal{G},\, \mathcal{L}_{\mathcal{C}})$ if for any $\mathbf{t}\in D_{\mathbf{T}}$ and $\mathbf{s} \in D_\mathbf{S}$, $P^{\mathcal{M}}_{\mathbf{t}}(\mathbf{s})$ is uniquely computable from $P^{\mathcal{M}}(\mathbf{N})$ in any SEM $\mathcal{M}$ with causal graph $\mathcal{G}$ and compatible with $\mathcal{L}_{\mathcal{C}}$ such that $P^{\mathcal{M}}(\mathbf{n})>0$ for any $\mathbf{n} \in D_{\mathbf{N}}$.
    \end{definition}
    Next example demonstrates a scenario in which a non-identifiable causal effect from $\mathcal{G}$ becomes identifiable from $(\mathcal{G},\, \mathcal{L}_{\mathcal{C}})$.
    \begin{example}\label{ex:example2}
    Consider the DAG in Figure \ref{fig: example1} in which $D_T=\{0,1\}$. 
    It is known that $P_x(y)$ is not identifiable from $\mathcal{G}$. 
    However, given  $\mathcal{C} = \{T\}$,  $\mathcal{L}_{\mathcal{C}}(0)= \{(Z,W)\}$, and $\mathcal{L}_{\mathcal{C}}(1)= \{(W,Y)\}$, $P_x(y)$ becomes identifiable and is equal to $P(y)$. 
    \end{example}

    \paragraph{Problem description:}
    Suppose we are given the DAG $\mathcal{G}$ and the label set $\mathcal{L}_{\mathcal{C}}$, where $\mathcal{C}\subseteq \Lambda_{\mathcal{G}}$ is the set of control variables. 
    This paper seeks to determine when $P_{\mathbf{t}}(\mathbf{S})$ is identifiable from $(\mathcal{G}, \mathcal{L}_{\mathcal{C}})$.    

\section{MAXIMAL-REGULAR LABELS}
    \cite{pensar2015labeled} first introduced the notion of \emph{maximality} and \emph{regularity} for a certain class of graphs called labeled DAGs. 
    Herein, we extend these notions for $(\mathcal{G}, \mathcal{L}_{\mathcal{C}})$ and show how to modify $\mathcal{G}$ and $\mathcal{L}_{\mathcal{C}}$ such that $\mathcal{L}_{\mathcal{C}}$ becomes \emph{maximal-regular} with respect to $\mathcal{G}$.
    
    \begin{definition}
         $\mathcal{L}_{\mathcal{C}}$ is regular w.r.t. DAG $\mathcal{G}$ if for any $(Y,X) \in \mathcal{L}_{\mathcal{C}}$, $\mathcal{C}^X$ is not empty, where $\mathcal{C}^X$ denotes $Pa_X \cap \mathcal{C}$.
    \end{definition}
    We now show that \textit{redundant} edges from $\mathcal{G}$ can be removed such that $\mathcal{L}_\mathcal{C}$ will be regular w.r.t. the new graph.
    \begin{lemma} \label{lem: regularity}
        If $(Y, X) \in \mathcal{L}_{\mathcal{C}}(\mathbf{c})$ and $\mathcal{C}^X = \varnothing$, then  
        $X\independent$ ${Y}\vert {Pa_X\setminus\{Y\}}$ and $(Y,X)$ can be deleted from $\mathcal{G}$.
    \end{lemma}
    To prove Lemma \ref{lem: regularity}, we first prove the following lemma.
    \begin{lemma} \label{lem: CSI proof} 
        If $(Y, X) \in \mathcal{L}_{\mathcal{C}}(\mathbf{c})$, then 
        \[
        X\independent Y\vert \mathcal{C}=\mathbf{c}, Pa_X\setminus(\{Y\}\cup\mathcal{C}).
        \]
    \end{lemma}
    \begin{proof}
        Let $\mathcal{H}$ be the DAG obtained by removing $(Y,X)$ from $\mathcal{G}$, and $\mathcal{M}_{\mathbf{c}}$ be the SEM over $\mathbf{V \setminus \mathcal{C}}$ when $do(\mathcal{C} = \mathbf{c})$.
        As $(Y, X) \in \mathcal{L}_{\mathcal{C}}(\mathbf{c})$, $\mathcal{H}$ is a causal graph for $\mathcal{M}_{\mathbf{c}}$. 
        Moreover, since $\mathcal{C}\subseteq \Lambda_{\mathcal{G}}$, we have
        \begin{equation} \label{eq: proof lemma CSI}
            P^{\mathcal{M}_{\mathbf{c}}}(\mathbf{V} \setminus \mathcal{C}) = P^{\mathcal{M}}(\mathbf{V} \setminus \mathcal{C} \vert \mathcal{C}=\mathbf{c}).
        \end{equation}
        Because $(Y,X) \in \mathbf{E}$ and $\mathcal{G}$ is a DAG, $X \notin Anc_Y$. Thus, local Markov property (\cite{pearl2009causality}) for $\mathcal{M}_{\mathbf{c}}$ implies that  
        $
            X \independent_{P^{\mathcal{M}_{\mathbf{c}}}} Y\vert  Pa_X\setminus(\{Y\}\cup\mathcal{C}).
        $
        Combining this with Equation \eqref{eq: proof lemma CSI} implies that $X\independent Y\vert \mathcal{C}=\mathbf{c}, Pa_X\setminus(\{Y\}\cup\mathcal{C})$.
    \end{proof} 
    Next, we prove Lemma \ref{lem: regularity}. 
    \begin{proof}
        Let $\mathbf{A} = Pa_X \setminus \{Y\}$. 
        Since $(Y, X) \in \mathcal{L}_{\mathcal{C}}(\mathbf{c})$, Lemma \ref{lem: CSI proof} implies that for any $\mathbf{a}\in D_{\mathbf{A}}$ and $y_1,y_2 \in D_Y$, 
        \[
            P(X\vert \mathbf{c}, \mathbf{a},y_1)= P(X\vert \mathbf{c}, \mathbf{a},y_2).
        \]
        Since $\dsep{X}{\mathcal{C}}{Pa_X}{\mathcal{G}}$, Rule 1 of do-calculus implies that 
        $P(X\vert \mathbf{c}, \mathbf{a},y_i) = P(X\vert \mathbf{a},y_i)$, for $i\in \{1,2\}$. Thus, 
        \[
            P(X\vert \mathbf{a},y_1)= P(X\vert \mathbf{a},y_2).
        \]
        Hence, $X \independent Y \vert \mathbf{A}$.
    \end{proof}
    
   Lemma \ref{lem: regularity} implies that if an edge $(Y,X)$ belongs to $\mathcal{L}_\mathcal{C}$ and $\mathcal{C}^X=\varnothing$, then $(Y,X)$ can be removed from the causal graph, i.e., $\mathcal{G}'=(\mathbf{V},\mathbf{E}\setminus\{(Y,X)\})$ is still a causal graph for the joint distribution $P(\mathbf{V})$. 
    Therefore, by removing all such edges from $\mathcal{G}$, we obtain a DAG such that for all $(Y,X) \in \mathcal{L}_\mathcal{C}$, $\mathcal{C}^X$ is not empty. 
    
    \begin{example}\label{ex:example3}
    Consider the causal DAG in Figure \ref{fig: example1} with control variable $\mathcal{C} = \{T\}$, where $D_T=\{0,1\}$ and label set $\mathcal{L}_\mathcal{C}=(\mathcal{L}_\mathcal{C}(0),\mathcal{L}_\mathcal{C}(1))$, where $\mathcal{L}_{\mathcal{C}}(0)= \{(Z,W)\}$ and $\mathcal{L}_{\mathcal{C}}(1)= \{(W,Y),(X,Z)\}$. This label set is not regular w.r.t. $\mathcal{G}$ because $(X,Z)\in\mathcal{L}_{\mathcal{C}}(1)$ but $\mathcal{C}^Z = \varnothing$. However, by removing edge $(X,Z)$ from the DAG, $\mathcal{L}_{\mathcal{C}}$ becomes regular w.r.t. the new DAG.
    \end{example}
    
    The following lemma shows that only the realizations of $\mathcal{C}^X$, e.g., $\mathbf{c}^X$, are relevant to determine whether an edge $(Y, X)$ belongs to  $\mathcal{L}_{\mathcal{C}}$ or not.
    \begin{lemma} \label{lem: maximality}
        If $(Y, X) \in \mathcal{L}_{\mathcal{C}}(\mathbf{c})$ and $\mathcal{C}^X \neq \varnothing$, then 
        $ (Y, X) \in \mathcal{L}_{\mathcal{C}^X}(\mathbf{c}^X)$.
    \end{lemma}
    \begin{proof}
        Suppose $\mathbf{A} = Pa_X \setminus (\mathcal{C}^X \cup \{Y\})$. 
        Since $(Y, X) \in \mathcal{L}_{\mathcal{C}}(\mathbf{c})$, Lemma \ref{lem: CSI proof} implies that for any $\mathbf{a}\in D_{\mathbf{A}}$ and $y_1,y_2 \in D_Y$, 
        \[
            P_{\mathbf{c}}(X\vert \mathbf{a}, y_1) = P_{\mathbf{c}}(X\vert \mathbf{a}, y_2).
        \]
        Because variables in $\mathcal{C}$ have no parents then $\mathcal{G}_{\overline{\mathcal{C}^X}, \overline{\mathcal{C} \setminus \mathcal{C}^X (Pa_X \setminus \mathcal{C}^X)}}=\mathcal{G}$ and 
        $\dsep{X}{\mathcal{C} \setminus \mathcal{C}^X}{Pa_X }{\mathcal{G}}$. Hence, 
        rule 3 of do-calculus implies that $P_{\mathbf{c}}(X\vert Pa_X \setminus \mathcal{C}^X) = P_{\mathbf{c}^X}(X\vert Pa_X\setminus \mathcal{C}^X)$.
        By combining the above equations, we obtain
        \begin{align*}
            P_{\mathbf{c}^X}(X\vert \mathbf{a},y_1) 
            = P_{\mathbf{c}^X}(X\vert \mathbf{a},y_2).
        \end{align*}
        This concludes that $ (Y, X) \in \mathcal{L}_{\mathcal{C}^X}(\mathbf{c}^X)$.
    \end{proof}
    \begin{corollary} \label{cor: maximality}
        Suppose $(Y, X) \in \mathcal{L}_{\mathcal{C}}(\mathbf{c}_1)$ for some $\mathbf{c}_1\in D_\mathcal{C}$. Let $\mathbf{c}_2\in D_\mathcal{C}$ such that $\mathbf{c}_1\neq \mathbf{c}_2$ but $\mathbf{c}_1^{\mathbf{X}}=\mathbf{c}_2^{\mathbf{X}}$, then we can add $(Y, X)$ to $\mathcal{L}_{\mathcal{C}}(\mathbf{c}_2)$.
    \end{corollary}
    
    \begin{definition}[Maximal-regular]
        We say $\mathcal{L}_{\mathcal{C}}$ is \emph{maximal} w.r.t. $\mathcal{G}$ if we cannot add any new label to $\mathcal{L}_{\mathcal{C}}$ using Corollary \ref{cor: maximality}. Label set $\mathcal{L}_{\mathcal{C}}$ is called maximal-regular w.r.t. $\mathcal{G}$ if it is both maximal and regular w.r.t. $\mathcal{G}$.
    \end{definition}
     
    \begin{example}
    Consider the causal graph in Figure \ref{fig: example1} with $D_T=\{0,1\}$, $\mathcal{C} = \{T\}$, $\mathcal{L}_{\mathcal{C}}(0)= \{(Z,W)\}$ and  $\mathcal{L}_{\mathcal{C}}(1)= \{(W,Y)\}$.
    In this case, $\mathcal{L}_\mathcal{C}$ is maximal-regular w.r.t. this causal graph. 
    \end{example}


\section{MAIN RESULT} \label{sec: main result}
    Next theorem states that the identifiablity of $P_{\mathbf{t}}(\mathbf{s})$ given $(\mathcal{G}, \mathcal{L}_{\mathcal{C}})$ is equivalent to the identifiability of $P_{\mathbf{t}}(\mathbf{s})$ from a list of DAGs. Note that \cite{tian2003ID} introduced a sound and complete algorithm for checking whether $P_{\mathbf{t}}(\mathbf{s})$ is identifiable from a DAG. Therefore, using the next theorem, we can develop a sound and complete algorithm for identifiability of $P_{\mathbf{t}}(\mathbf{s})$ from $(\mathcal{G}, \mathcal{L}_{\mathcal{C}})$.

    \begin{theorem} \label{thm: main}
        Suppose $\mathcal{G}=(\mathbf{V},\mathbf{E})$ is a DAG with observable variables $\mathbf{N} \subseteq \mathbf{V}$, and let $\mathcal{C}\subseteq \Lambda_{\mathcal{G}}$ and $\mathbf{T},\mathbf{S} \subseteq \mathbf{N}\setminus \mathcal{C}$ be three disjoint subsets.
        Furthermore, suppose the set of labels $\mathcal{L}_{\mathcal{C}}$ is maximal-regular w.r.t. $\mathcal{G}$. 
        Causal effect of $\mathbf{T}$ on $\mathbf{S}$ is identifiable from $(\mathcal{G}, \mathcal{L}_{\mathcal{C}})$ 
        if and only if the causal effect of $\mathbf{T}$ on $\mathbf{S}$ is identifiable from 
        $\mathcal{G}_{\mathbf{c}}=(\mathbf{V}', \mathbf{E}'\setminus \mathcal{L}_{\mathcal{C}}(\mathbf{c}))$ 
        for every $\mathbf{c} \in D_{\mathcal{C}}$, where 
        $\mathbf{V}' = \mathbf{V} \setminus \mathcal{C}$ and $\mathbf{E}' = \mathbf{E} \cap (\mathbf{V}' \times \mathbf{V}')$.
    \end{theorem}
    \begin{proof}
        \emph{Sufficiency:}
        Suppose $\mathcal{M}$ is a SEM with causal graph $\mathcal{G}$ and compatible with label set $\mathcal{L}_{\mathcal{C}}$.
        For every $\mathbf{c}\in D_{\mathcal{C}}$, we construct a SEM $\mathcal{M}_{\mathbf{c}}$ over $\mathbf{V}'$ by setting $\mathcal{C}$ to be $\mathbf{c}$ in the equations of $\mathcal{M}$. 
        Hence, $\mathcal{G}_{\mathbf{c}}=(\mathbf{V}', \mathbf{E}'\setminus \mathcal{L}_{\mathcal{C}}(\mathbf{c}))$ is a causal graph for $\mathcal{M}_{\mathbf{c}}$ and 
        \begin{equation} \label{eq: new observation dist}
        P^{\mathcal{M}_{\mathbf{c}}}(\mathbf{N}\setminus \mathcal{C})
        = P^{\mathcal{M}}_\mathbf{c}(\mathbf{N} \setminus \mathcal{C})
        = P^{\mathcal{M}}(\mathbf{N} \setminus \mathcal{C} \mid \mathbf{c}).
        \end{equation}
        The above relations hold due to the fact that $\mathcal{C} \subseteq \Lambda_{\mathcal{G}}$, that is, intervened variables do not have parents.
        As $P_{\mathbf{t}}(\mathbf{S})$ is identifiable from $\mathcal{G}_{\mathbf{c}}$, 
        $P^{\mathcal{M}_{\mathbf{c}}}_{\mathbf{t}}(\mathbf{S})$ is uniquely computable from 
        $P^{\mathcal{M}_{\mathbf{c}}}(\mathbf{N}\setminus \mathcal{C})$, and therefore, from 
        $P^{\mathcal{M}}(\mathbf{N})$.
        From the definition of $\mathcal{M}_{\mathbf{c}}$, we have
        \[ 
            P^{\mathcal{M}_{\mathbf{c}}}_{\mathbf{t}}(\mathbf{S})
            = P^{\mathcal{M}}_{\mathbf{c}, \mathbf{t}}(\mathbf{S})
            = P^{\mathcal{M}}_{\mathbf{t}}(\mathbf{S} \mid \mathbf{c}).
        \]
        Hence, $P^{\mathcal{M}}_{\mathbf{t}}(\mathbf{S} \mid \mathbf{c})$ is uniquely computable from $P^{\mathcal{M}}(\mathbf{N})$ for every $\mathbf{c}\in D_{\mathcal{C}}$. 
        On the other hand, for any $\mathbf{t}\in D_{\mathbf{T}}$, we have
        \begin{equation} \label{eq: if part}
        \begin{split}
            P^{\mathcal{M}}_{\mathbf{t}}(\mathbf{S})
            &= \sum_{\mathbf{c}\in D_{\mathcal{C}}} P^{\mathcal{M}}_{\mathbf{t}}(\mathbf{S} \mid \mathbf{c}) P^{\mathcal{M}}_{\mathbf{t}}(\mathbf{c}) \\
            &= \sum_{\mathbf{c}\in D_{\mathcal{C}}} P^{\mathcal{M}}_{\mathbf{t}}(\mathbf{S} \mid \mathbf{c}) P^{\mathcal{M}}(\mathbf{c}). 
        \end{split}
        \end{equation}
        Equation \eqref{eq: if part} implies that $P^{\mathcal{M}}_{\mathbf{t}}(\mathbf{S})$ is uniquely computable from $P^{\mathcal{M}}(\mathbf{N})$ and the causal effect of $\mathbf{T}$ on $\mathbf{S}$ is identifiable from $(\mathcal{G}, \mathcal{L}_{\mathcal{C}})$. Note that the functional that maps $P^{\mathcal{M}}(\mathbf{N})$ to $P^{\mathcal{M}}_{\mathbf{t}}(\mathbf{S})$ in \eqref{eq: if part} is continuous. 
        
        \emph{Necessity:}
        For any $\mathbf{c}_0\in D_{\mathcal{C}}$, we need to show that $P_{\mathbf{t}}(\mathbf{S})$ is identifiable from $\mathcal{G}_{\mathbf{c}_0}$ knowing that $P_{\mathbf{t}}(\mathbf{S})$ is identifiable from $(\mathcal{G}, \mathcal{L}_{\mathcal{C}})$. 
        We use proof by contradiction. 
        Suppose there exists two SEMs $\mathcal{M}_1$ and $\mathcal{M}_2$ over $\mathbf{V}'$ with causal graph $\mathcal{G}_{\mathbf{c}_0}$ such that $P^{\mathcal{M}_1}(\mathbf{N \setminus }\mathcal{C}) = P^{\mathcal{M}_2}(\mathbf{N \setminus} \mathcal{C})>0$, but there exists $\mathbf{t}\in D_{\mathbf{T}}$ and $\mathbf{s} \in D_\mathbf{S}$ such that $P^{\mathcal{M}_1}_{\mathbf{t}}(\mathbf{s}) \neq P^{\mathcal{M}_2}_{\mathbf{t}}(\mathbf{s})$.
        Let $\delta = |P^{\mathcal{M}_1}_{\mathbf{t}}(\mathbf{s})- P^{\mathcal{M}_2}_{\mathbf{t}}(\mathbf{s})|>0$. 
        We now construct two SEMs $\mathcal{M}_i^{\epsilon}$ for $i\in \{1,2\}$ over $\mathbf{V}$, where $\epsilon>0$ is a small number.
        
        First, we define $P^{\mathcal{M}_i^{\epsilon}}(\mathcal{C}):= \prod_{C\in \mathcal{C}} P^{\mathcal{M}_i^{\epsilon}}(C)$.  $\{P^{\mathcal{M}_i^{\epsilon}}(C): C\in\mathcal{C}\}$ are selected such that $P^{\mathcal{M}_i^{\epsilon}}(C=c)$ is positive for all $c\in D_C$
        and $P^{\mathcal{M}_i^{\epsilon}}(\mathbf{c}_0) = 1 - \epsilon$. Next, for each $X \in \mathbf{V \setminus } \mathcal{C}$ we define the equation of $X$ in $\mathcal{M}_i^{\epsilon}$ as follows. 
        \begin{equation}
            X = 
            \begin{cases}
                X^{\mathcal{M}_i} & \text{ if } \mathcal{C}^X = \mathbf{c_0}^X \\
                \sim \text{uniform in } D_X & \text{ otherwise, } 
            \end{cases}
        \end{equation}
        where $X^{\mathcal{M}_i}$ indicates the random variable $X$ in SEM $\mathcal{M}_i$.
        Note that $\mathcal{M}_i^{\epsilon}$ is compatible with $\mathcal{G}$ since $X$ is a function of $\mathcal{C}^X$ and parents of $X$ in $\mathcal{G}_{\mathbf{c}_0}$.
        Next, we show that $\mathcal{M}_i^{\epsilon}$ is also compatible with $\mathcal{L}_{\mathcal{C}}(\mathbf{c})$ for each $\mathbf{c} \in D_{\mathcal{C}}$, and therefore, compatible with $\mathcal{L}_{\mathcal{C}}$. 
        For each $X\in \mathbf{V \setminus}\mathcal{C}$, if $\mathbf{c}^{X} \neq \mathbf{c}_0^{X}$, then $X$ is not a function of any other variables in $\mathcal{M}_i^{\epsilon}$. 
        If $\mathbf{c}^{X} = \mathbf{c}_0^{X}$, then $X = X^{\mathcal{M}_i}$.
        In this case, for each edge $(Y, X)$ in $\mathcal{G}_{\mathbf{c}_0}$, $(Y, X) \notin \mathcal{L}_{\mathcal{C}}(\mathbf{c})$ since $\mathcal{L}_{\mathcal{C}}$ is maximal and $(Y, X) \notin \mathcal{L}_{\mathcal{C}}(\mathbf{c}_0)$.
        Hence, $\mathcal{M}_i^{\epsilon}$ is compatible with $(\mathcal{G},\mathcal{L}_{\mathcal{C}})$.
    
        From the definition of $\mathcal{M}_i^{\epsilon}$ and the fact that  $\mathcal{C}\subseteq \Lambda_{\mathcal{G}}$, we obtain
        \begin{equation} \label{eq: two models 1}
            P^{\mathcal{M}_i}(\mathbf{N \setminus }\mathcal{C}) 
            = P^{\mathcal{M}_i^{\epsilon}}(\mathbf{N \setminus }\mathcal{C} \mid \mathbf{c}_0),
        \end{equation}
        \begin{equation} \label{eq: two models 2}
            P^{\mathcal{M}_i}_{\mathbf{t}}(\mathbf{S}) 
            = P^{\mathcal{M}_i^{\epsilon}}_{\mathbf{t}}(\mathbf{S} \mid \mathbf{c}_0).
        \end{equation}
        We can write $P^{\mathcal{M}_i^{\epsilon}}(\mathbf{N})$ as 
        $$P^{\mathcal{M}_i^{\epsilon}}(\mathbf{N \setminus} \mathcal{C}, \mathbf{c}_0)\mathds{1}_{\{\mathcal{C}=\mathbf{c_0}\}} + P^{\mathcal{M}_i^{\epsilon}}(\mathbf{N}) \mathds{1}_{\{\mathcal{C}\mathbf{ \neq c_0}\}}.$$
        The second term is not larger than $P^{\mathcal{M}_i^{\epsilon}}(\mathcal{C}\mathbf{\neq c_0})$ which is equal to $\epsilon$.
        On the other hand, from \eqref{eq: two models 1}, we obtain
        \begin{equation} \label{eq: Nw0 to N}
        \begin{split}
            P^{\mathcal{M}_i^{\epsilon}}(\mathbf{N \setminus} \mathcal{C}, \mathbf{c_0}) 
            &= P^{\mathcal{M}_i^{\epsilon}}(\mathbf{c_0}) P^{\mathcal{M}_i^{\epsilon}}(\mathbf{N\setminus} \mathcal{C} \mid \mathbf{c_0})\\
            &= (1-\epsilon) P^{\mathcal{M}_i}(\mathbf{N \setminus }\mathcal{C}) .
        \end{split}
        \end{equation}
        Therefore, we have
        \begin{equation} \label{eq: bound 1}
            0 \leq 
            P^{\mathcal{M}_i^{\epsilon}}(\mathbf{N}) 
            - (1-\epsilon) P^{\mathcal{M}_i}(\mathbf{N \setminus }\mathcal{C})\mathds{1}_{\{\mathcal{C}=\mathbf{c_0}\}}
            \leq \epsilon.
        \end{equation}
        Since $P^{\mathcal{M}_1}(\mathbf{N \setminus }\mathcal{C})
        = P^{\mathcal{M}_2}(\mathbf{N \setminus }\mathcal{C})$, Equation \eqref{eq: bound 1} implies  
        \begin{equation} \label{eq: continuous 1}
            |P^{\mathcal{M}_1^{\epsilon}}(\mathbf{N})
            - P^{\mathcal{M}_2^{\epsilon}}(\mathbf{N})| 
            \leq \epsilon.
        \end{equation}
        Note that $P^{\mathcal{M}_i^{\epsilon}}_{\mathbf{t}}(\mathbf{c}) = P^{\mathcal{M}_i^{\epsilon}}(\mathbf{c}) $, since $\mathcal{C} \subseteq \Lambda_{\mathcal{G}}$. 
        Hence,
        \begin{equation} \label{eq: pts}
        \begin{split}
            P^{\mathcal{M}_i^{\epsilon}}_{\mathbf{t}}(\mathbf{S}) &
            = \sum_{\mathbf{c}} P^{\mathcal{M}_i^{\epsilon}}_{\mathbf{t}}( \mathbf{S} \mid \mathbf{c}) P^{\mathcal{M}_i^{\epsilon}}(\mathbf{c}) \\
            & \leq (1-\epsilon)  P^{\mathcal{M}_i^{\epsilon}}_{\mathbf{t}}( \mathbf{S} \mid \mathbf{c}_0) + \epsilon.
        \end{split}
        \end{equation}
        Equations \eqref{eq: two models 2} and \eqref{eq: pts} imply that
        \begin{equation} \label{eq: bound 2}
            0 \leq
            P^{\mathcal{M}_i^{\epsilon}}_{\mathbf{t}}(\mathbf{S}) 
            - (1-\epsilon) P^{\mathcal{M}_i}_{\mathbf{t}}(\mathbf{S}) 
            \leq \epsilon.
        \end{equation}
        Because $\delta = |P^{\mathcal{M}_1}_{\mathbf{t}}(\mathbf{s})- P^{\mathcal{M}_2}_{\mathbf{t}}(\mathbf{s})|$, Equation \eqref{eq: bound 2} implies  
        \begin{equation} \label{eq: continuous 2}
            |P^{\mathcal{M}_1^{\epsilon}}_{\mathbf{t}}(\mathbf{S})
            - P^{\mathcal{M}_2^{\epsilon}}_{\mathbf{t}}(\mathbf{S})|
            \geq (1-\epsilon)\delta -2\epsilon.
        \end{equation}
        Recall that $\mathcal{M}_i^{\epsilon}$ is compatible with $(\mathcal{G},\mathcal{L}_{\mathcal{C}})$ and $P^{\mathcal{M}_i^{\epsilon}}(\mathbf{n})>0$ for any $\mathbf{n}\in D_{\mathbf{N}}$. 
        Hence, there exists a continuous operator $\mathcal{F}$ such that $\mathcal{F}(P^{\mathcal{M}_i^{\epsilon}}(\mathbf{N})) = P^{\mathcal{M}_i^{\epsilon}}_{\mathbf{t}}(\mathbf{S})$.
        Note that $\mathcal{F}$ does not depend on $\epsilon$. Also, $\delta$ is a constant number.
        In this case, Equations \eqref{eq: continuous 1} and \eqref{eq: continuous 2} for sufficiently small $\epsilon$ contradict with the continuity of $\mathcal{F}$.
    \end{proof}
    
    As the proof of the sufficiency in Theorem \ref{thm: main} is constructive, it allows us to develop an algorithm (Algorithm \ref{algo}) for causal effect identification from $(\mathcal{G}, \mathcal{L}_{\mathcal{C}})$.
    
    \begin{algorithm}[h]
        \caption{Causal effect ID from $(\mathcal{G}, \mathcal{L}_{\mathcal{C}})$.}
        \label{algo}
        \begin{algorithmic}[1]
            \STATE \textbf{Input:} $\mathcal{G} = (\mathbf{V,E}),\, \mathcal{L}_{\mathcal{C}}$
            \STATE \textbf{Output:} A formula for $P_{\mathbf{t}}(\mathbf{S})$ based on $P(\mathbf{N})$, or return non-identifiable.
            
            \tcp{\textit{To make regular:}}
            \WHILE{$\exists (Y,X)\in \mathcal{L}_{\mathcal{C}}$ such that $\mathcal{C}^X=\varnothing$}
                \STATE Delete $(Y,X)$ from $\mathcal{G}$
            \ENDWHILE
            \tcp{\textit{To make maximal:}}
            \WHILE{$\exists \mathbf{c}_1, \mathbf{c}_2 \in D_{\mathcal{C}}, (Y, X) \in \mathcal{L}_{\mathcal{C}}(\mathbf{c}_1)$ such that\\ $\mathbf{c}_1^{\mathbf{X}}=\mathbf{c}_2^{\mathbf{X}}$ and $(Y, X) \notin \mathcal{L}_{\mathcal{C}}(\mathbf{c}_2)$}
                \STATE Add $(Y, X)$ to $\mathcal{L}_{\mathcal{C}}(\mathbf{c}_2)$
            \ENDWHILE
            \STATE $\mathbf{V}' \gets \mathbf{V} \setminus \mathcal{C}$
            \STATE $\mathbf{N}' \gets \mathbf{N} \setminus \mathcal{C}$
            \STATE $\mathbf{E}' \gets \mathbf{E} \cap (\mathbf{V}' \times \mathbf{V}')$
            \FOR{$\mathbf{c}\in D_{\mathcal{C}}$}
                \STATE $\mathcal{G}_{\mathbf{c}} \gets (\mathbf{V}', \mathbf{E}'\setminus \mathcal{L}_{\mathcal{C}}(\mathbf{c}))$
                \IF{$P_{\mathbf{t}}(\mathbf{S})$ is not ID from $\mathcal{G}_{\mathbf{c}}$}
                    \STATE \textbf{Return} Non-identifiable.
                \ELSE
                    \STATE $P^{\mathbf{c}}(\mathbf{N}') \gets P(\mathbf{N} \setminus \mathcal{C} \mid \mathbf{c})$
                    \STATE $\mathbf{F_c} \gets $ A formula for $P_{\mathbf{t}}(\mathbf{S})$ based on $P^{\mathbf{c}}(\mathbf{N}')$ using graph $\mathcal{G}_{\mathbf{c}}$
                \ENDIF
            \ENDFOR
            \STATE \textbf{Return} $\sum_{\mathbf{c}\in D_{\mathcal{C}}} \mathbf{F_c} P(\mathbf{c}). $
        \end{algorithmic}
    \end{algorithm}

    Algorithm \ref{algo} regularizes the labels in lines 3-4 and makes them maximal in lines 5-6 using Lemma \ref{lem: regularity} and Corollary \ref{cor: maximality}, respectively.
    Then, based on Theorem \ref{thm: main}, for each $\mathbf{c} \in D_{\mathcal{C}}$ it checks the identifiability of $P_{\mathbf{t}}(\mathbf{S})$ in DAG $\mathcal{G}_{\mathbf{c}} = (\mathbf{V}', \mathbf{E}'\setminus \mathcal{L}_{\mathcal{C}}(\mathbf{c}))$.
    If for any $\mathbf{c}$, this is non-identifiable, the algorithm terminates and outputs non-identifiable. 
    Otherwise, it defines the distribution $P^{\mathbf{c}}$ over $\mathbf{N}'$ in line 15 based on Equation \eqref{eq: new observation dist}. 
    Since $P_{\mathbf{t}}(\mathbf{S})$ is identifiable from $\mathcal{G}_{\mathbf{c}}$, the algorithm finds a formula for $P_{\mathbf{t}}(\mathbf{S})$ in terms of  $P^{\mathbf{c}}(\mathbf{N}')$ in line 16. 
    This step can be performed using any do-calculus based method.
    Finally, Algorithm \ref{algo} outputs a formula for $P_{\mathbf{t}}(\mathbf{S})$ in terms of $P(\mathbf{N})$ in line 17, using Equation \eqref{eq: if part}. 
    \begin{corollary}
        Algorithm \ref{algo} for the problem of causal effect identification from $(\mathcal{G}, \mathcal{L}_{\mathcal{C}})$ is sound and complete.
    \end{corollary}

    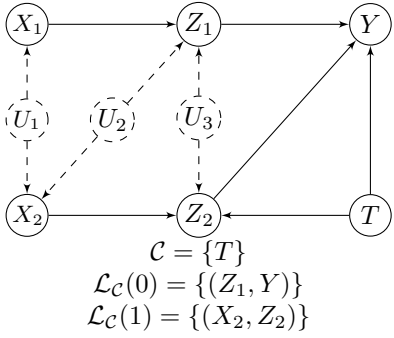
\begin{figure}[t]
        \centering
        \tikzstyle{block} = [draw, fill=white, circle, text width=1.1em, text centered, align=center]
    	\tikzstyle{input} = [coordinate]
    	\tikzstyle{output} = [coordinate]
        \begin{subfigure}[b]{0.3\textwidth}
            \centering
        	\begin{tikzpicture}[->, node distance=1cm,>=latex', every node/.style={inner sep=1pt}]
        	    \node[block](X1){\small $X_1$};
    	        \node[block,dashed](U1)[below =0.7cm of X1]{\small $U_1$};
    	        \node[block](X2)[below =0.7cm of U1]{\small $X_2$};
    	        \node[block](Z1)[right= 1.7 cm of X1]{$Z_1$};
    	        \node[block,dashed](U2)[right= 0.55 cm of U1]{$U_2$};
    	        \node[block](Z2)[right= 1.7 cm of X2]{$Z_2$};
                \node[block,dashed](U3)[above= 0.7 cm of Z2]{$U_3$};
    	        \node[block](Y)[right= 1.7 cm of Z1]{$Y$};
    	        \node[block](T)[right= 1.7 cm of Z2]{$T$};
                \draw (X1) to (Z1);
                \draw (X2) to (Z2);
                \draw (Z1) to (Y);
                \draw (Z2) to (Y);
                \draw[dashed] (U1) to (X1);
                \draw[dashed] (U1) to (X2);
                \draw[dashed] (U2) to (X2);
                \draw[dashed] (U2) to (Z1);
                \draw[dashed] (U3) to (Z1);
                \draw[dashed] (U3) to (Z2);
                \draw (T) to (Y);
                \draw (T) to (Z2);
    		\end{tikzpicture}
        \end{subfigure}
    	\begin{subfigure}[b]{0.2\textwidth}
    	    \centering
    	    $\mathcal{C} = \{T\}$
    		$\mathcal{L}_{\mathcal{C}}(0)= \{(Z_1,Y)\}$ $\mathcal{L}_{\mathcal{C}}(1)= \{(X_2,Z_2)\}$
    		\hspace{1cm}
    	\end{subfigure}
    	\caption{Causal graph and $\mathcal{L}_{\mathcal{C}}$ in Example \ref{ex:exampl5}.}
    	\label{fig: example 3}
    \end{figure}

\begin{example}\label{ex:exampl5}
    Consider the causal graph with control variable $T$, $D_T=\{0,1\}$, and the label set $\mathcal{L}_\mathcal{C}$ depicted in Figure \ref{ex:example3}. 
    Note that $\mathcal{L}_\mathcal{C}$ is maximal-regular. 
    In this case, the causal effect of $\{X_1,X_2\}$ on $Y$, i.e., $P_{x_1,x_2}(Y)$ is non-identifiable from the graph alone. 
    However, it is identifiable when we have access to both the graph and the label set $\mathcal{L}_\mathcal{C}$.
    In this case, Algorithm \ref{algo} finds the following formula.
    \begin{align*}
        &P_{x_1,x_2}(Y)=\mathbf{F}_0P(T=0)+\mathbf{F}_1P(T=1),
    \end{align*}
    where
    \[
        \mathbf{F}_0 = \sum_{Z_2}P(Y \mid Z_2,T=0)P(Z_2 \mid X_2, T=0), 
    \]
    \[
        \mathbf{F}_1=\!\!\sum_{Z_1,Z_2} P(Y \mid Z_1,Z_2,T=1)P(Z_1, Z_2 \mid X_1, T=1).
    \]
\end{example}
    
\section{LEARNING $\mathcal{L}_{\mathcal{C}}$ FROM $P(\mathbf{N})$}
    In the previous section, we proposed an algorithm for causal effect identification, when a set of labels $\mathcal{L}_{\mathcal{C}}$ are given as side information. 
    In this section, we propose Algorithm \ref{algo 2} for inferring these labels from observational distribution.
    \begin{algorithm}[h]
        \caption{Learning $\mathcal{L}_{\mathcal{C}}$ from $P(\mathbf{N})$.}
        \label{algo 2}
        \begin{algorithmic}[1]
            \STATE \textbf{Input:} $\mathcal{G}=(\mathbf{V}, \mathbf{E})$, $P(\mathbf{N})$
            \STATE \textbf{Output:} $\mathcal{L}_{\mathcal{C}}$
            \FOR{$\mathbf{c}\in D_{\mathcal{C}}$}
                \FOR{$X,Y \!\in \mathbf{N}$ such that $(Y,X)\in \mathbf{E}$}
                        \IF{$X \independent Y \vert \mathbf{c}, Anc_{\{X,Y\}}\cap \mathbf{N}\setminus \mathcal{C}$}
                            \STATE Add $(Y,X)$ to $\mathcal{L}_{\mathcal{C}}(\mathbf{c})$
                        \ENDIF
                \ENDFOR
            \ENDFOR
            \STATE \textbf{Return} $\mathcal{L}_{\mathcal{C}}$
        \end{algorithmic}
    \end{algorithm}

    In order to add an edge $(Y,X)\in \mathbf{E}$ to $\mathcal{L}_{\mathcal{C}}(\mathbf{c})$, Algorithm \ref{algo 2} evaluates the following CSI in line 5.
    \begin{equation} \label{eq: CSI for learning Lc}
        X \independent Y \vert \mathbf{c}, Anc_{\{X,Y\}}\cap \mathbf{N}\setminus \mathcal{C}.
    \end{equation}
    Note that the conditioning set in Equation \eqref{eq: CSI for learning Lc} is a subset of observed variables.
    Hence, we can evaluate this CSI from $P(\mathbf{N})$ when $X$ and $Y$ are observed. 

    If Equation \eqref{eq: CSI for learning Lc} holds for $(Y,X) \in \mathbf{E}$, by Definition \ref{def:label}, $(Y,X) \in \mathcal{L}_\mathcal{C}(\mathbf{c})$, which proves the soundness of the algorithm. 
    Although the reverse does not hold in general, we present a graphical constraint in the following theorem, under which the reverse also holds true, and consequently, the algorithm is complete. 
    
    \begin{theorem} \label{thm: learning Lc}
        Suppose $(Y,X)\in \mathbf{E}$ such that $X$ is not a function of $Y$ when $\mathcal{C}=\mathbf{c}$. 
        Let $\mathcal{H}$ be the DAG obtained by removing $(Y,X)$ from $\mathcal{G}$.
        If I) $X,Y \in \mathbf{N}$ and II) $X$ and $Y$ are d-separable in $\mathcal{H}$ given subsets of $\mathbf{N}$, then Algorithm \ref{algo 2} correctly add $(Y,X)$ to $\mathcal{L}_{\mathcal{C}}(\mathbf{c})$.
    \end{theorem}
    To prove this theorem, we use the following definition and Lemma by \cite{verma1991equivalence}.
    \begin{definition}[Inducing path]
        A path $\mathcal{P}$ between two observed variables $X$ and $Y$ is called an inducing path if I) every non-endpoint observed vertex in $\mathcal{P}$ is a collider\footnote{A vertex $V_i$ is called a collider on a path $(V_1, \dots V_k)$ if $1<i<k$ and $V_{i-1} \to V_i \gets V_{i+1}$.}, and II) every collider in $\mathcal{P}$ belongs to $Anc_{\{X,Y\}}$. 
    \end{definition}
    \begin{lemma}[\cite{verma1991equivalence}] \label{lem: IP iff}
        There is no inducing path between $X$ and $Y$ in $\mathcal{H}$ if and only if $\dsep{X}{Y}{(Anc_{\{X,Y\}})\cap \mathbf{N}}{\mathcal{H}}$.
    \end{lemma}
    Now, we prove Theorem \ref{thm: learning Lc}. 
    \begin{proof}
        $X$ and $Y$ are d-separable in $\mathcal{H}$ by subsets of $\mathbf{N}$. 
        This is equivalent to saying there is no inducing path between $X$ and $Y$ in $\mathcal{H}$ (\cite{verma1991equivalence}). 
        Also, Lemma \ref{lem: IP iff} presents a necessary and sufficient graphical condition for this constraint. Hence, Equation \eqref{eq: CSI for learning Lc} holds and Algorithm \ref{algo 2} correctly adds $(Y,X)$ to $\mathcal{L}_{\mathcal{C}}(\mathbf{c})$ in line 6. 
    \end{proof}

The following example demonstrates a scenario where Algorithm \ref{algo 2} can infer a label set from the observational distribution such that a non-identifiable causal effect (using only the information from the causal graph) becomes identifiable (using both the graph and the inferred label set). 

\begin{example}
    Consider again the example in Figure \ref{fig: example 3}. 
    Since I) $Z_1,Y\in\mathbf{N}$ and II) $Z_2$ d-separates $Z_1$ and $Y$ in the resulting graph after removing $(Z_1,Y)$, Theorem \ref{thm: learning Lc} implies that Algorithm \ref{algo 2} can infer $\mathcal{L}_\mathcal{C}(0)$. 
    Similarly, I) $X_2,Z_2\in\mathbf{N}$ and II) $X_2$ and $Z_2$ are d-separated by empty set in the resulting graph after removing $(X_2,Z_2)$.
    Hence, Algorithm \ref{algo 2} can infer $\mathcal{L}_\mathcal{C}(1)$, and therefore, $\mathcal{L}_\mathcal{C}$.
    As we showed in Example \ref{ex:exampl5}, Algorithm \ref{algo} can then identify the causal effect of $\{X_1,X_2\}$ on $Y$ while this causal effect is non-identifiable from the graph alone. 
\end{example}

\section{COMPLEXITY}
    Algorithm \ref{algo} determines the identifiability of $P_{\mathbf{t}}(\mathbf{s})$ from $(\mathcal{G},\mathcal{L}_\mathcal{C})$ by checking the identifiability of $P_{\mathbf{t}}(\mathbf{s})$ in $\mathcal{G}_{\mathbf{c}} = (\mathbf{V}', \mathbf{E}'\setminus \mathcal{L}_{\mathcal{C}}(\mathbf{c}))$ for each $\mathbf{c} \in D_{\mathcal{C}}$. 
    On the other hand, \cite{tian2003ID} proposed a sound and complete algorithm to check the identifiability of a causal effect from a causal graph. 
    The complexity of this algorithm is polynomial in the number of vertices of the graph. 
    Therefore, Algorithm \ref{algo} determines the identifiability of $P_{\mathbf{t}}(\mathbf{s})$ from $(\mathcal{G},\mathcal{L}_\mathcal{C})$ by calling \cite{tian2003ID}'s algorithm at most $|D_{\mathcal{C}}|$ number of times.
    Consequently, when the number of control variables does not grow by the number of vertices in the causal graph (e.g., $|\mathcal{C}|=r$ for some fixed $r$), complexity of Algorithm \ref{algo} will also be polynomial in the number of vertices of the graph. 

    In the next example, we compare our algorithm with the method in \cite{tikka2020identifying}.
    \begin{figure}[t]
        \centering
        \tikzstyle{block} = [draw, fill=white, circle, text width=1.1em, text centered, align=center]
    	\tikzstyle{input} = [coordinate]
    	\tikzstyle{output} = [coordinate]
        \begin{subfigure}[b]{0.3\textwidth}
            \centering
        	\begin{tikzpicture}[->, node distance=1cm,>=latex', every node/.style={inner sep=1pt}]
    	        \node[block,dashed](U){\small $U$};
    	        \node[block](X)[below left = 1.5cm of U]{\small $X$};
    	        \node[block](Y)[below right = 1.5cm of U]{\small $Y$};
    	        \node[block](C)[below right=1.5cm of X]{\small $C$};
                \draw (X) to (Y);
                \draw[dashed] (U) to (X);
                \draw[dashed] (U) to (Y);
                \draw (C) to (Y);
                \draw (C) to (X);
    		\end{tikzpicture}
        \end{subfigure}
    	\begin{subfigure}[b]{0.2\textwidth}
    	    \centering
    	    $\mathcal{C} = \{C\}$
    		$\mathcal{L}_{\mathcal{C}}(0)= \{(U,X)\}$ $\mathcal{L}_{\mathcal{C}}(1)= \{(X,Y)\}$
    		\hspace{1cm}
    	\end{subfigure}
    	\caption{Causal graph and $\mathcal{L}_{\mathcal{C}}$ in Example \ref{ex:exampl6}.}
    	\label{fig: example 5}
    \end{figure}
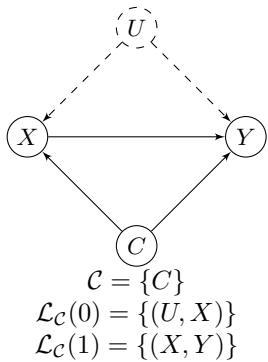

    \begin{example} \label{ex:exampl6}
        Consider the causal graph with control variable $C$, $D_C=\{0,1\}$, and the label set $\mathcal{L}_\mathcal{C}$ depicted in Figure \ref{ex:example3}. 
        Herein, the goal is to identify the causal effect of $X$ on $Y$.
        In this case, Algorithm \ref{algo} finds the following formula. 
        \begin{equation} \label{eq: ex7}
            P_{x}(Y)=\mathbf{F}_0 P(C=0) + \mathbf{F}_1 P(C=1),
        \end{equation}
        where
        \[
            \mathbf{F}_0 = P(Y \mid X,C=0), 
        \]
        \[
            \mathbf{F}_1 = P(Y \mid C=1).
        \]
        Algorithm \ref{algo} finds $\mathbf{F}_0$ and $\mathbf{F}_1$ by solving only two identifiablity problems from two graphs with vertices $\{X,Y\}$.
        On the other hand, \cite{tikka2020identifying} provide eight rules ($R1-R8$) that can be applied recursively to compute $P_{x}(Y)$ from $P(X,Y,C)$. 
        In this example, there are many ways that these rules can be applied. 
        For example, $(R2 \to R1 \to R1\to R7\to R4\to R6\to R2)$ or $(R2 \to R7 \to R3\to R1\to R1\to R4\to R6 \to R2)$ (See Figure $4$ in \cite{tikka2020identifying} for more details.)
        Applying the eight rules with these specific orders will result in the same formula as in Equation \eqref{eq: ex7}. 
        It is worthy to note that finding a correct order for applying these rules is challenging and can be computationally expensive. 
        Therefore, our algorithm finds the solution faster than the method in \cite{tikka2020identifying}.
    \end{example}

\section{EXPERIMENTS}
    \begin{figure*}[t] 
        \centering
        \captionsetup{justification=centering}
        \begin{subfigure}[b]{0.45\textwidth}
            \centering
            \includegraphics[width=0.8\textwidth]{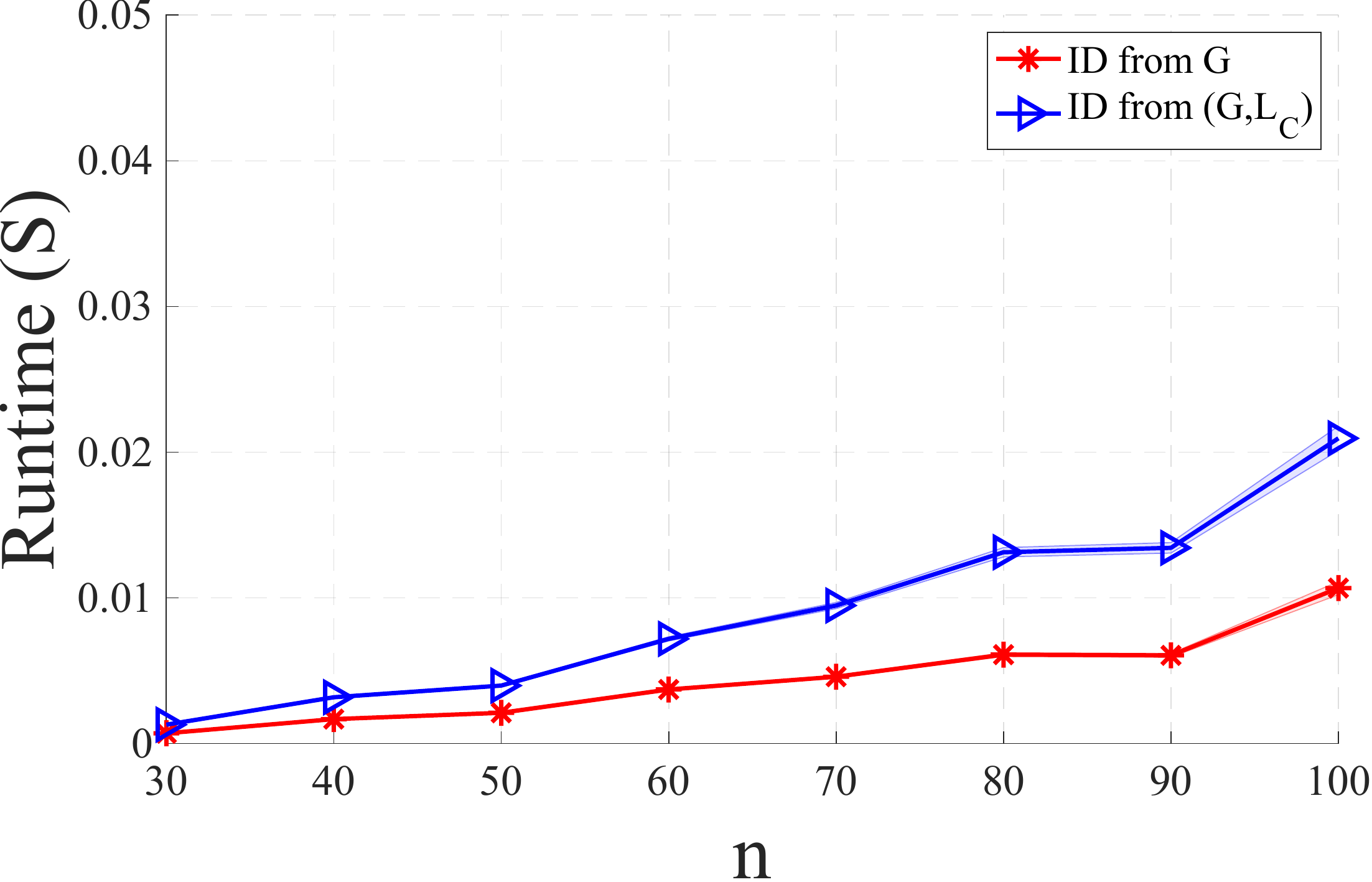}
            \caption{}
            \label{fig: exp 1}
        \end{subfigure}
        \begin{subfigure}[b]{0.45\textwidth}
            \centering
            \includegraphics[width=0.8\textwidth]{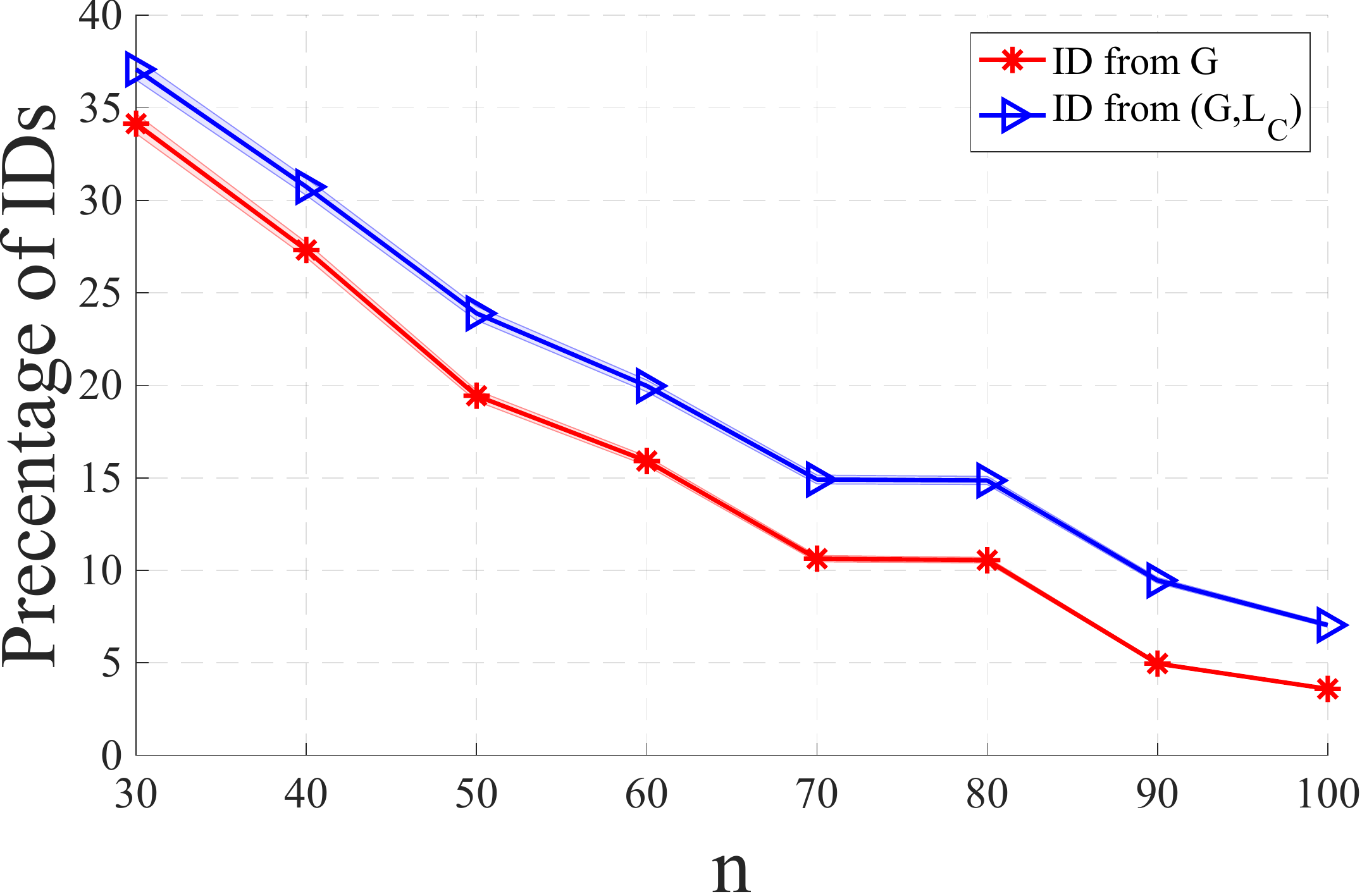}
            \caption{}
            \label{fig: exp 2}
        \end{subfigure}
        \caption{Performance of Algorithm \ref{algo} and the identification algorithm by \cite{tian2003ID} over random graphs generated from Erdos-Renyi models. }
        \label{fig: experiments}
    \end{figure*}

    The MATLAB implementation of Algorithm \ref{algo} along with our implementation of the algorithm by \cite{tian2003ID} are publicly available\footnote{https://github.com/Ehsan-Mokhtarian/causalID}.
    In this section, we illustrate the performance of these methods over a set of random graphs in Figure \ref{fig: experiments}.
    Next, we describe the settings of this experiment. 
    
    The variables are assumed to be binary. 
    The skeleton of the graphs is generated from Erdos-Renyi model $G(n,p)$ (\cite{erdHos1960evolution}), where $n$ is the number of variables and $p$ is the probability of an edge. 
    $n$ is chosen in a wide range from 30 to 100, and $p=\frac{\log(n)}{n}$. 
    After constructing the skeleton, the edges are oriented based on a random ordering over the vertices. 
    Each variable belongs to the set of observed variables with probability $0.7$. 
    Also, each variable in $\Lambda_{\mathcal{G}}$ is randomly selected as a control variable with probability $0.8$.
    To make the model simpler, we replace the control variables $\mathcal{C}$ with one variable $C$, such that the children of $C$ are the union of children of the variables in $\mathcal{C}$. 
    To build the labels, for each edge $(Y,X) \in \mathbf{E}$ such that $C \in Pa_X$, either we add $(Y,X)$ to no $\mathcal{L}_C(c)$ with probability $0.2$, or otherwise, we add it to one of the labels $\mathcal{L}_C(c)$ uniformly at random.
    Each point on the plots is reported as the average of 1000 runs, and the shaded areas indicate the $80\%$ confidence intervals.
    We randomly split $\mathbf{N}$ into two subsets $\mathbf{S}$ and $\mathbf{T}$. 
    Then, we run both algorithms to identify the causal effect $P_{\mathbf{t}}(\mathbf{s})$.
    
    Figure \ref{fig: exp 1} shows the runtime of the algorithms in seconds. As this figure suggests, both of these algorithms are practically fast and are scalable to large graphs. 
    Figure \ref{fig: exp 2} shows the percentage of the runs in which the corresponding algorithm manages to identify the causal effect. 
    This figure shows that given the label set $\mathcal{L}_{\mathcal{C}}$, more causal effects will be identifiable than the case where only the causal graph is given. 
    
    We did not report the results of the algorithm by \cite{tikka2020identifying} in this figure since their method is not scalable to large graphs. 
    For instance, even for the graphs of size 10, their algorithm sometimes requires 30 minutes of runtime to terminate, while our method requires around 0.02 seconds to terminate for the graphs with 100 variables.

\section{CONCLUSION}
    We studied the causal effect identification problem when extra side information about the underlying generative causal model in the form of CSI relations is available. 
    To this end, we showed that when CSI relations of control variables are given, the identifiability of an interventional distribution from observational distribution is equivalent to a series of causal effect identifications only from causal graphs.  
    Since there exist sound and complete algorithms for causal effect identification only from a causal graph, we could develop a sound and complete algorithm for causal effect identification in the presence of CSI relations. 
    Although, such CSI relations in general cannot be inferred from observational distribution, we introduce a graphical constraint under which CSI relations of control variables can be inferred from the observation distribution.
    
\subsubsection*{Acknowledgements}
    This research was in part supported by the Swiss National Science Foundation under NCCR Automation, grant agreement 51NF40\_180545 and Swiss SNF project 200021\_204355/1.
\bibliography{bibliography}
\end{document}